\documentclass{article}

\usepackage{times}
\usepackage{amsmath,amsthm,amssymb}
\usepackage{bm}
\usepackage{amsfonts}
\usepackage{algorithm}
\usepackage{algorithmic}
\usepackage{graphicx}
\usepackage{hyperref}

\newtheorem{theorem}{Theorem}[section]
\newtheorem{lemma}[theorem]{Lemma}

\newcommand{\tr}{\mathop{ \rm tr}}

\title{Revisiting k-means: New Algorithms via Bayesian Nonparametrics
}
\author{Brian Kulis$^1$ and Michael I. Jordan$^2$\\$^1$Department of CSE, Ohio State University\\$^2$Department of EECS and Department of Statistics, UC Berkeley}

\begin{document}
\maketitle

\begin{abstract}
Bayesian models offer great flexibility for clustering applications---Bayesian nonparametrics can be used for modeling infinite mixtures, and hierarchical Bayesian models can be utilized for sharing clusters across multiple data sets.  For the most part, such flexibility is lacking in classical clustering methods such as k-means.  In this paper, we revisit the k-means clustering algorithm from a Bayesian nonparametric viewpoint.  Inspired by the asymptotic connection between k-means and mixtures of Gaussians, we show that a Gibbs sampling algorithm for the Dirichlet process mixture approaches a hard clustering algorithm in the limit, and further that the resulting algorithm monotonically minimizes an elegant underlying k-means-like clustering objective that includes a penalty for the number of clusters.  We generalize this analysis to the case of clustering multiple data sets through a similar asymptotic argument with the hierarchical Dirichlet process.  We also discuss further extensions that highlight the benefits of our analysis: i) a spectral relaxation involving thresholded eigenvectors, and ii) a normalized cut graph clustering algorithm that does not fix the number of clusters in the graph.
\end{abstract}

\section{Introduction}
There is now little debate that Bayesian statistics have had tremendous impact on the field of machine learning.  For the problem of clustering, the topic of this paper, the Bayesian approach allows for flexible models in a variety of settings.  For instance, Latent Dirichlet Allocation~\cite{lda}, a hierarchical mixture of multinomials, reshaped the topic modeling community and has become a standard tool in document analysis.  Bayesian nonparametric models, such as the Dirichlet process mixture~\cite{hjort}, result in infinite mixture models which do not fix the number of clusters in the data upfront; these methods continue to gain popularity in the learning community.

Yet despite the success and flexibility of the Bayesian framework, simpler methods such as k-means remain the preferred choice in many large-scale applications.  For instance, in visual bag-of-words models~\cite{vbow}, large collections of image patches are quantized, and k-means is universally employed for this task.  A major motivation for using k-means is its simplicity and scalability: whereas Bayesian models require sampling algorithms or variational inference techniques which can be difficult to implement and are often not scalable, k-means is straightforward to implement and works well for a variety of applications.

In this paper, we attempt to achieve the best of both worlds by designing scalable hard clustering algorithms from a Bayesian nonparametric viewpoint.  Our approach is inspired by the connection between k-means and mixtures of Gaussians, namely that the k-means algorithm may be viewed as a limit of the expectation-maximization (EM) algorithm---if all of the covariance matrices corresponding to the clusters in a Gaussian mixture model are equal to $\sigma I$ and we let $\sigma$ go to zero, the EM steps approach the k-means steps in the limit.  As we will show, in the case of a Dirichlet process (DP) mixture model---the standard Bayesian nonparametric mixture model---we can perform a similar limiting argument in the context of a simple Gibbs sampler.  This leads to an algorithm with hard cluster assignments which is very similar to the classical k-means algorithm except that a new cluster is formed whenever a point is sufficiently far away from all existing cluster centroids.  Further, we show that this algorithm monotonically converges to a local optimum of a k-means-like objective which includes a penalty for the number of clusters.

We then take a step further into the realm of hierarchical Bayesian models, and extend our analysis to the hierarchical Dirichlet process (HDP)~\cite{teh}.  The HDP is a model for shared clusters across multiple data sets; when we take an analogous asymptotic limit for the HDP mixture, we obtain a novel k-means-like algorithm that clusters multiple data sets with shared cluster structure.  The resulting algorithm clusters each data set into local clusters, but local clusters are shared across data sets to form global clusters.  The underlying objective function in this case turns out to be the k-means objective with additional penalties for the number of local clusters and the number of global clusters.  

To further demonstrate the practical value of our approach, we present two additional extensions.  First, we show that there is a spectral relaxation for the k-means-like objective arising from the DP mixture.  Unlike the standard relaxation for k-means, which computes the top-k eigenvectors, our relaxation involves computing eigenvectors corresponding to eigenvalues above a threshold, and highlights an interesting connection between spectral methods and Bayesian nonparametrics.  Second, given existing connections between k-means and graph clustering, we propose a penalized normalized cut objective for graph clustering, and utilize our earlier results to design an algorithm for monotonic optimization.  Unlike the standard normalized cut formulation~\cite{norm_cut,multiclass_ncut}, our formulation does not fix the number of clusters in the graph.
We conclude with some results highlighting that our approach retains the flexibility of the Bayesian models while featuring the scalability of the classical techniques.  Ultimately, we hope that this line of work will inspire additional research on the integration of Bayesian nonparametrics and hard clustering methods. 

\section{Background}
We begin with a short discussion of the relevant models and algorithms considered in this work: mixtures of Gaussians, k-means, and DP mixtures.  

\subsection{Gaussian Mixture Models and k-means}
In a (finite) Gaussian mixture model, we assume that data arises from the following distribution:
\begin{displaymath}
p(\bm{x}) = \sum_{c=1}^k \pi_c {\mathcal N}(\bm{x}~|~\bm{\mu}_c,\Sigma_c),
\end{displaymath}
where $k$ is the fixed number of components, $\pi_c$ are the mixing coefficients, and $\bm{\mu}_c$ and $\Sigma_c$ are the means and covariances, respectively, of the $k$ Gaussian distributions.  In the non-Bayesian setting, we can use the EM algorithm to perform maximum likelihood given a set of observations $\bm{x}_1, ..., \bm{x}_n$.  
Briefly, we initialize the means $\bm{\mu}_c$, covariances $\Sigma_c$, and mixing coefficients $\pi_c$.  Then we alternate between the E-step and M-step.  In the E-step, using the current parameter values, we compute the following quantities for all $i=1,...,n$ and for all $c=1,...,k$:
\begin{displaymath}
\gamma(z_{ic}) = \frac{\pi_c {\mathcal N}(\bm{x}_i~|~\bm{\mu}_c,\Sigma_c)}{\sum_{j=1}^c \pi_j {\mathcal N}(\bm{x}_i~|~\bm{\mu}_j, \Sigma_j)}.
\end{displaymath}
In the M-step, we re-estimate the parameters using the values of $\gamma(z_{ic})$:
\begin{eqnarray*}
\bm{\mu}_c^{new} & = & \frac{1}{n_c} \sum_{i=1}^n \gamma(z_{ic}) \bm{x}_i\\
\Sigma_c^{new} & = & \frac{1}{n_c} \sum_{i=1}^n \gamma(z_{ic}) (\bm{x}_i - \bm{\mu}_c^{new})(\bm{x}_i - \bm{\mu}_c^{new})^T\\
\pi_c^{new} & = & \frac{n_c}{n},
\end{eqnarray*}
where $n_c = \sum_{i=1}^n \gamma(z_{ic})$.  One can show that the EM algorithm converges to a local optimum of the log likelihood function.  Note that the values $\gamma(z_{ic})$ are the probabilities of assigning point $\bm{x}_i$ to cluster $c$, and so the resulting clustering is a soft clustering of the data.

A related model for clustering is provided by the k-means objective function, an objective for discovering a hard clustering of the data.  Given a set of data points $\bm{x}_1, ..., \bm{x}_n$, the k-means objective function attempts to find clusters $\ell_1, ..., \ell_k$ to minimize the following objective function:
\begin{eqnarray*}
\min_{\{\ell_c\}_{c=1}^k} & \sum_{c=1}^k \sum_{\bm{x} \in \ell_c} \|\bm{x} - \bm{\mu}_c\|_2^2\\
\mbox{where} & \bm{\mu}_c = \frac{1}{|\ell_c|} \sum_{\bm{x} \in \ell_c} \bm{x}.
\end{eqnarray*}
The most popular method for minimizing this objective function is simply called the k-means algorithm.  One initializes the algorithm with a hard clustering of the data along with the cluster means of these clusters.  Then the algorithm alternates between reassigning points to clusters and recomputing the means.  For the reassignment step one computes the squared Euclidean distance from each point to each cluster mean, and finds the minimum, by computing
$\ell^*(i) = \mbox{argmin}_c \|\bm{x}_i - \bm{\mu}_c\|_2^2.$
Each point is then reassigned to the cluster indexed by $\ell^*(i)$.  The centroid update step of the algorithm then recomputes the mean of each cluster, updating $\bm{\mu}_c$ for all $c$.

The EM algorithm for mixtures of Gaussians is quite similar to the k-means algorithm.  Indeed, one can show a precise connection between the two algorithms.  Suppose in the mixture of Gaussians model that all Gaussians have the same fixed covariance equal to $\sigma I$.  Because they are fixed, the covariances need not be re-estimated during the M-step.  In this case, the E-step takes the following form:
\begin{displaymath}
\gamma(z_{ic}) = \frac{\pi_c \cdot \mbox{exp} \big ( -\frac{1}{2 \sigma} \|\bm{x}_i - \bm{\mu}_c\|_2^2 \big )}{\sum_{j=1}^k \pi_j \cdot \mbox{exp} \big ( -\frac{1}{2 \sigma} \|\bm{x}_i - \bm{\mu}_j\|_2^2 \big )},
\end{displaymath}
It is straightforward to show that, in the limit as $\sigma \to 0$, the value of $\gamma(z_{ic})$ approaches zero for all $c$ except for the one corresponding to the smallest distance $\|\bm{x}_i - \bm{\mu}_c\|_2^2$.  In this case, the E-step is equivalent to the reassignment step of k-means, and one can further easily show that the M-step exactly recomputes the means of the new clusters, establishing the equivalence of the updates.  
We also note that further interesting connections between k-means and probabilistic clustering models were explored in~\cite{kurihara}.  Though they approach the problem differently (i.e., not from an asymptotic view), the authors also ultimately obtain k-means-like algorithms that can be applied in the nonparametric setting.

\subsection{Dirichlet Process Mixture Models}
We briefly review DP mixture models~\cite{hjort}.  We can equivalently write the standard Gaussian mixture as a generative model  where one chooses a cluster with probability $\pi_c$ and then generates an observation from the Gaussian corresponding to the chosen cluster.  The distribution over the cluster indicators follows a discrete distribution, so a Bayesian extension to the mixture model arises by first placing a Dirichlet prior of dimension $k$, $\mbox{Dir}(k, \bm{\pi}_0)$, on the mixing coefficients, for some $\bm{\pi}_0$.  If we further assume that the covariances of the Gaussians are fixed to $\sigma I$ and that the means are drawn from some prior distribution $G_0$, we obtain the following Bayesian model:
\begin{eqnarray*}
\bm{\mu}_1, ..., \bm{\mu}_k & \sim & G_0\\
\bm{\pi} & \sim & \mbox{Dir}(k,\bm{\pi}_0)\\
\bm{z}_1, ..., \bm{z}_n & \sim & \mbox{Discrete}(\bm{\pi})\\
\bm{x}_1, ..., \bm{x}_n & \sim & {\mathcal N}(\bm{\mu}_{\bm{z}_i},\sigma I),
\end{eqnarray*}
letting $\bm{\pi} = (\pi_1, ..., \pi_k)$.  One way to view the DP mixture model is to take a limit of the above model as $k \to \infty$ when choosing $\bm{\pi}_0 = (\alpha/k) \bm{e}$, where $\bm{e}$ is the vector of all ones.  
One of the simplest algorithms for inference in a DP mixture is based on Gibbs sampling; this approach was utilized by~\cite{west} and further discussed by~\cite{neal_dp}, Algorithm 2.  The state of the underlying Markov chain consists of the set of all cluster indicators and the set of all cluster means.  The algorithm proceeds by first looping repeatedly through each of the data points and performing Gibbs moves on the cluster indicators for each point.  For $i = 1, ..., n$, we reassign $\bm{x}_i$ to existing cluster $c$ with probability
$n_{-i,c} \cdot {\mathcal N}(\bm{x}_i~|~\bm{\mu}_c,\sigma I) / Z,$
where $n_{-i,c}$ is the number of data points (excluding $\bm{x}_i$) that are assigned to cluster $c$.  With probability
\begin{displaymath}
\frac{\alpha}{Z} \int {\mathcal N}(\bm{x}_i~|~\bm{\mu},\sigma I) d G_0(\bm{\mu}),
\end{displaymath}
we start a new cluster.  $Z$ is an appropriate normalizing constant.  If we end up choosing to start a new cluster, we select its mean from the posterior distribution obtained from the prior $G_0$ and the single sample $\bm{x}_i$.  
After resampling all clusters, we perform Gibbs moves on the means: we sample $\bm{\mu}_c$ given all points currently assigned to cluster $c, \forall c$.

We note that one often writes the DP mixture model (adapted to our Gaussian mixture scenario) as follows:
\begin{eqnarray*}
\begin{array}{llll}
G & \sim & \mbox{DP}(\alpha, G_0) &\\
\phi_i & \sim & G & \mbox{for $i = 1, ..., n$}\\
\bm{x}_i & \sim & {\mathcal N}(\phi_i,\sigma I) & \mbox{for $i = 1, ..., n$.}
\end{array}
\end{eqnarray*}
Thus, each $G$ is a draw from the Dirichlet process $\mbox{DP}(G_0, \alpha)$, whose base measure $G_0$ is a prior over means of the Gaussians.  We can think of a draw from $G$ as choosing one of the infinite means $\bm{\mu}_c$ drawn from $G_0$, with the property that the means are chosen with probability equal to the corresponding mixing weights.  As a result, each $\phi_i$ is equal to $\bm{\mu}_c$ for some $c$.

\section{Hard Clustering via Dirichlet Processes}
In the following sections, we derive hard clustering algorithms based on DP mixture models.  We will analyze properties of the resulting algorithms and show connections to existing hard clustering algorithms, particularly k-means.

\subsection{Asymptotics of the DP Gibbs Sampler}
Using the DP mixture model introduced in the previous section, let us first define $G_0$.  Since we are fixing the covariances, $G_0$ is the prior distribution over the means, which we will take to be a zero-mean Gaussian with $\rho I$ covariance, i.e., $\bm{\mu} \sim {\mathcal N}(\bm{0},\rho I)$.  
Given this prior, the Gibbs probabilities can be computed in closed form.  A straightforward calculation reveals that the probability of starting a new cluster is equal to:
\begin{displaymath}
\frac{\alpha}{Z} (2 \pi (\rho + \sigma))^{-d/2} \cdot \mbox{exp} \bigg (-\frac{1}{2(\rho + \sigma)} \|\bm{x}_i\|^2 \bigg ).
\end{displaymath}
Similarly, the probability of being assigned to cluster $c$ is equal to
\begin{displaymath}
\frac{n_{-i,c}}{Z} (2 \pi \sigma)^{-d/2} \mbox{exp} \bigg ( - \frac{1}{2 \sigma} \|\bm{x}_i - \bm{\mu}_c\|_2^2 \bigg ).
\end{displaymath}
$Z$ normalizes these probabilities to sum to 1.  We now would like to see what happens to these probabilities as $\sigma \to 0$.  However, in order to obtain non-trivial assignments, we must additionally let $\alpha$ be a function of $\sigma$ and $\rho$.  In particular, we will write
$\alpha = (1 + \rho/\sigma)^{d/2} \cdot \mbox{exp}(-\frac{\lambda}{2 \sigma})$
for some $\lambda$.  Now, let $\hat{\gamma}(z_{ic})$ correspond to the posterior probability of point $i$ being assigned to cluster $c$ and let $\hat{\gamma}(z_{i,new})$ be the posterior probability of starting a new cluster.
After simplifying, we obtain the following probabilities to be used during Gibbs sampling: $\hat{\gamma}(z_{ic}) = $
\begin{displaymath}
\frac{n_{-i,c} \cdot \mbox{exp} (-\frac{1}{2 \sigma}\|\bm{x}_i - \bm{\mu}_c\|^2 )}{\mbox{exp} \big (-\frac{\lambda}{2 \sigma} - \frac{\|\bm{x}_i\|^2}{2(\rho+\sigma)} \big ) + \sum_{j=1}^k n_{-i,j} \cdot \mbox{exp} (-\frac{1}{2 \sigma} \|\bm{x}_i - \bm{\mu}_j\|^2 ) }
\end{displaymath}
for existing clusters and $\hat{\gamma}(z_{i,new}) = $
\begin{displaymath}
\frac{\mbox{exp} \big (-\frac{\lambda}{2 \sigma} - \frac{\|\bm{x}_i\|^2}{2(\rho+\sigma)} \big )}{\mbox{exp} \big (-\frac{\lambda}{2 \sigma} - \frac{\|\bm{x}_i\|^2}{2(\rho+\sigma)} \big ) + \sum_{j=1}^k n_{-i,j} \cdot \mbox{exp} (-\frac{1}{2 \sigma} \|\bm{x}_i - \bm{\mu}_j\|^2 ) }
\end{displaymath}
for generating a new cluster.  
Now we consider the asymptotic behavior of the above probabilities. 
The numerator for $\hat{\gamma}(z_{i,new})$ can be written as
\begin{displaymath}
\mbox{exp} \bigg (-\frac{1}{2 \sigma} \bigg [\lambda + \frac{\sigma}{\rho+\sigma} \|\bm{x}_i\|^2 \bigg ] \bigg ).
\end{displaymath}  
It is straightforward to see that, as $\sigma \to 0$ with a fixed $\rho$, the $\lambda$ term dominates this numerator.  Furthermore, all of the above probabilities will become binary; in particular, the values of $\hat{\gamma}(z_{i,c})$ and $\hat{\gamma}(z_{i,new})$ will be increasingly dominated by the smallest value of $\{\|\bm{x}_i - \bm{\mu}_1\|^2, ..., \|\bm{x}_i - \bm{\mu}_k\|^2, \lambda\}$.  In the limit, only the smallest of these values will receive a non-zero $\hat{\gamma}$ value.
The resulting update, therefore, takes a simple form that is analogous to the k-means cluster reassignment step.  We reassign a point to the cluster corresponding to the closest mean, \textit{unless} the closest cluster has squared Euclidean distance greater than $\lambda$.  In this case, we start a new cluster.

If we choose to start a new cluster, the final step is to sample a new mean from the posterior based on the prior $G_0$ and the single observation $\bm{x}_i$.  Similarly, once we have performed Gibbs moves on the cluster assignments, we must perform Gibbs moves on all the means, which amounts to sampling from the posterior based on $G_0$ and all observations in a cluster.  Since the prior and likelihood are Gaussian, the posterior will be Gaussian as well.  If we let $\bar{\bm{x}}_c$ be the mean of the points currently assigned to cluster $c$ and $n_c$ be the number of points assigned to cluster $c$, then the posterior is a Gaussian with mean $\tilde{\bm{\mu}}_c$ and covariance $\tilde{\Sigma}_c$, where
\begin{displaymath}
\tilde{\bm{\mu}}_c = \bigg ( 1 + \frac{\sigma}{\rho n_c} \bigg)^{-1} \bar{\bm{x}}_c,~~\tilde{\Sigma}_c = \frac{\sigma \rho}{\sigma + \rho n_c} I.
\end{displaymath}
As before, we consider the asymptotic behavior of the above Gaussian distribution as $\sigma \to 0$.  The mean of the Gaussian approaches $\bar{\bm{x}}_c$ and the covariance goes to zero, meaning that the mass of the distribution becomes concentrated at $\bar{\bm{x}}_c$.  Thus, in the limit we choose $\bar{\bm{x}}_c$ as the mean.

\renewcommand{\algorithmicrequire}{\textbf{Input:}}
\renewcommand{\algorithmicensure}{\textbf{Output:}}

\begin{algorithm}[t] \small
\centering
\caption{DP-means \label{algo:npmeans}}
\begin{algorithmic}
\begin{minipage}{.9 \textwidth}
\REQUIRE $\bm{x}_1, ..., \bm{x}_n$: input data, $\lambda:$ cluster penalty parameter
\ENSURE Clustering $\ell_1, ..., \ell_k$ and number of clusters $k$
\vspace*{5pt}
\STATE 1.~Init. $k=1, \ell_1 = \{\bm{x}_1, ..., \bm{x}_n\}$ and $\bm{\mu}_1$ the global mean.
\STATE 2.~Init. cluster indicators $z_i = 1$ for all $i = 1,...,n$.
\STATE 3.~Repeat until convergence
\begin{itemize}
\item For each point $\bm{x}_i$
\begin{itemize}
\item Compute $d_{ic} = \|\bm{x}_i - \bm{\mu}_c\|^2$ for $c = 1, ..., k$
\item If $\min_c d_{ic} > \lambda$, set $k = k+1,$  $z_i = k,$ and $\bm{\mu}_k = \bm{x}_i$.
\item Otherwise, set $z_i = \mbox{argmin}_c d_{ic}$.
\end{itemize}
\item Generate clusters $\ell_1, ..., \ell_k$ based on $z_1, ..., z_k$:
$\ell_j = \{\bm{x}_i~|~z_i = j\}$.
\item For each cluster $\ell_j$, compute
$\bm{\mu}_j = \frac{1}{|\ell_j|}\sum_{\bm{x} \in \ell_j} \bm{x}$.
\end{itemize}
\end{minipage}
\end{algorithmic}
\end{algorithm}

Putting everything together, we obtain a hard clustering algorithm that behaves similarly to k-means with the exception that a new cluster is formed whenever a point is farther than $\lambda$ away from every existing cluster centroid.  We choose to initialize the algorithm with a single cluster whose mean is simply the global centroid; the resulting algorithm is specified as Algorithm~\ref{algo:npmeans}, which we denote as the DP-means algorithm.  Note that, unlike standard k-means, which depends on the initial clustering of the data, the DP-means algorithm depends on the order in which data points are processed.  One area of future work would consider adaptive methods for choosing an ordering.

\subsection{Underlying Objective and the AIC}
With the procedure from the previous section in hand, we can now analyze its properties.  A natural question to ask is whether there exists an underlying objective function corresponding to this k-means-like algorithm.
In this section, we show that the algorithm monotonically decreases the following objective at each iteration, where an iteration is defined as a complete loop through all data points to update all cluster assignments and means:
\begin{eqnarray}
\min_{\{\ell_c\}_{c=1}^k} & \sum_{c=1}^k \sum_{\bm{x} \in \ell_c} \|\bm{x} - \bm{\mu}_c\|^2 + \lambda k \nonumber\\
\mbox{where} & \bm{\mu}_c = \frac{1}{|\ell_c|} \sum_{\bm{x} \in \ell_c} \bm{x}.
\label{eqn:npobj}
\end{eqnarray}
This objective is simply the k-means objective function with an additional penalty based on the number of clusters.  The threshold $\lambda$ controls the tradeoff between the traditional k-means term and the cluster penalty term.  We can prove the following:
\begin{theorem}
Algorithm~\ref{algo:npmeans} monotonically decreases the objective given in~\eqref{eqn:npobj} until local convergence.
\end{theorem}
\begin{proof}
The proof follows a similar argument as the proof for standard k-means.  The reassignment step results in a non-increasing objective since the distance between a point and its newly assigned cluster mean never increases; for distances greater than $\lambda$, we can generate a new cluster and pay a penalty of $\lambda$ while still decreasing the objective.  Similarly, the mean update step results in a non-increasing objective since the mean is the best representative of a cluster in terms of the squared Euclidean distance.
The fact that the algorithm will converge locally follows from the fact that the objective function cannot increase, and that there are only a finite number of possible clusterings of the data.
\end{proof}
Perhaps unsurprisingly, this objective has been studied in the past in conjunction with the Akaike Information Criterion (AIC).  For instance,~\cite{manning} describe the above penalized k-means objective function with a motivation arising from the AIC.  Interestingly, it does not appear that algorithms have been derived from this particular objective function, so our analysis seemingly provides the first constructive algorithm for monotonic local convergence as well as highlighting the connections to the DP mixture model.
Finally, in the case of k-means, one can show that the complete-data log likelihood approaches the k-means objective in the limit as $\sigma \to 0$.  We conjecture that a similar result holds for the DP mixture model, which would indicate that our result is not specific to the particular choice of the Gibbs sampler.  

\section{Clustering with Multiple Data Sets}

\begin{figure}
\begin{center}
\includegraphics[width=12cm]{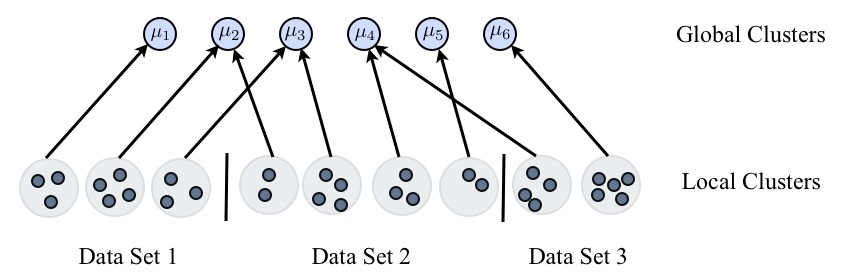}
\end{center}
\caption{Overview of clustering with multiple data sets.  Each data set has some number of local clusters, and each local cluster is associated with some global mean $\bm{\mu}_p$.  Each global mean $\bm{\mu}_p$ is computed as the mean of all points (across all data sets) associated with that global cluster.  When reassigning points to clusters, the objective function penalizes the formation of either a new local cluster or a new global cluster.  See text for details.}
\label{fig:hdp}
\end{figure}

One of the most useful extensions to the standard DP mixture model arises when we introduce another DP layer on top of the base measure.  Briefly, assume we have a set of data sets, each of which is modeled as a DP mixture.  However, instead of defining the base measure of each DP mixture using $G_0$, the prior over the means, we instead let $G_0$ itself be a Dirichlet process whose base measure is a prior over the means.   The result is that, given a collection of data sets, we can cluster each data set while ensuring that the clusters across the data sets are shared appropriately.
We will not describe the resulting hierarchical Dirichlet process (HDP) nor its corresponding sampling algorithms in detail, but we refer the reader to~\cite{teh} for a detailed introduction to the HDP model and a description of inference techniques.  We will see that the limiting process described earlier for the standard DP can be straightforwardly extended to the HDP; we will outline the algorithm below, and Figure~\ref{fig:hdp} gives an overview of the approach.

To set the stage, let us assume that we have $D$ data sets, $1, ..., j, ..., D$.  Denote $\bm{x}_{ij}$ to be data point $i$ from data set $j$, and let there be $n_j$ data points from each data set $j$.  The basic idea is that we will locally cluster the data points from each data set, but that some cluster means will be shared across data sets.  
Each data set $j$ has a set of local cluster indicators given by $z_{ij}$ such that $z_{ij} = c$ if data point $i$ in data set $j$ is assigned to local cluster $S_{jc}$.  Each local cluster $S_{jc}$ is associated to a global cluster mean $\bm{\mu}_{p}$.  

Recall the standard form for the DP mixture model under our settings:
\begin{eqnarray*}
\begin{array}{llll}
G & \sim & \mbox{DP}(\alpha, G_0) &\\
\phi_i & \sim & G & \mbox{for $i = 1, ..., n$}\\
\bm{x}_i & \sim & {\mathcal N}(\phi_i,\sigma I) & \mbox{for $i = 1, ..., n$.}
\end{array}
\end{eqnarray*}
For the HDP, we have a set of data sets indexed by $j$, each of which is a DP mixture.  However, instead of defining the base measure of each DP mixture using $G_0$, the prior over the means, we instead let $G_0$ itself be a Dirichlet process whose base measure is a prior over the means.  This yields the following:
\begin{eqnarray*}
\begin{array}{llll}
G_0 & \sim & \mbox{DP}(\gamma, H) &~\\
G_j & \sim & \mbox{DP}(\alpha, G_0) & \mbox{for $j = 1, ..., D$}\\
\phi_{ij} & \sim & G_j & \mbox{for all $i$, $j$}\\
\bm{x}_{ij} & \sim & {\mathcal N}(\phi_{ij},\sigma I) & \mbox{for all $i$, $j$.}
\end{array}
\end{eqnarray*}
Analogous to the standard DP mixture, the $\phi_{ij}$ chooses some global mean $\bm{\mu}_c$, now based on both the local and global Dirichlet processes $G_j$ and $G_0$, respectively.  The prior over the means of the Gaussian is now specified by $H$.

\subsection{The Hard Gaussian HDP}
We can  now extend the asymptotic argument that we employed for the hard DP algorithm to the HDP.  We will summarize the resulting algorithm; the derivation is analogous to the derivation for the single DP mixture case.
As with the hard DP algorithm, we will have a threshold that determines when to introduce a new cluster.  For the hard HDP, we will require two parameters: let $\lambda_{\ell}$ be the ``local" threshold parameter, and $\lambda_g$ be the ``global" threshold parameter.  The algorithm works as follows: for each data point $\bm{x}_{ij}$, we compute the distance to every global cluster $\bm{\mu}_p$.  For any global cluster $p$ for which there is no current association in data set $j$, we add a penalty of $\lambda_{\ell}$ to the distance (intuitively, this penalty captures the fact that if we end up assigning $\bm{x}_{ij}$ to a global cluster that is not currently in use by data set $j$, we will incur a penalty of $\lambda_{\ell}$ to create a new local cluster, which we only want to do if the cluster if sufficiently close to $\bm{x}_{ij}$).  We reassign each data point $\bm{x}_{ij}$ to its nearest cluster, unless the closest distance is greater than $\lambda_{\ell} + \lambda_g$, in which case we start a new global cluster (in this case we are starting a new local cluster and a new global cluster, hence the sum of the two penalties).  Then, for each local cluster, we consider whether to reassign it to a different global mean: for each local cluster $S_{jc}$, we compute the sum of distances of the points to every $\bm{\mu}_p$.  We reassign the association of $S_{jc}$ to the corresponding closest $\bm{\mu}_p$; if the closest is farther than $\lambda_g$ plus the sum of distances to the local cluster mean, then we start a new global cluster whose mean is the local mean.  Finally, we recompute all means $\bm{\mu}_p$ by computing the mean of all points (over all data sets) associated to each $\bm{\mu}_p$.  See Algorithm~\ref{algo:nphmeans} for the full specification of the procedure; the algorithm is derived directly as an asymptotic hard clustering algorithm based on the Gibbs sampler for the HDP.

\renewcommand{\algorithmicrequire}{\textbf{Input:}}
\renewcommand{\algorithmicensure}{\textbf{Output:}}

\begin{algorithm}[t] \small
\centering
\caption{Hard Gaussian HDP \label{algo:nphmeans}}
\begin{algorithmic}
\begin{minipage}{.9 \textwidth}
\REQUIRE $\{\bm{x}_{ij}\}$: input data, $\lambda_{\ell}:$ local cluster penalty parameter, $\lambda_g$: global cluster penalty parameter
\ENSURE Global clustering $\ell_1, ..., \ell_g$ and number of clusters $k_j$ for all data sets $j$
\vspace*{5pt}
\STATE 1.~Initialize $g = 1$, $k_j=1$ for all $j$ and $\bm{\mu}_1$ to be the global mean across all data sets.
\STATE 2.~Initialize local cluster indicators $z_{ij} = 1$ for all $i$ and $j$, and global cluster associations $v_{j1} = 1$ for all $j$.
\STATE 3.~Repeat steps 4-6 until convergence:
\STATE 4.~For each point $\bm{x}_{ij}$:
\begin{itemize}
\item Compute $d_{ijp} = \|\bm{x}_{ij} - \bm{\mu}_p\|^2$ for $p = 1, ..., g$.  
\item For all $p$ such that $v_{jc} \neq p$ for all $c = 1, ..., k_j$, set $d_{ijp} = d_{ijp} + \lambda_{\ell}$.
\item If $\min_p d_{ijp} > \lambda_{\ell} + \lambda_g$,
\begin{itemize}
\item Set $k_j = k_j + 1$, $z_{ij} = k_j$, $g = g + 1$, $\bm{\mu}_g = \bm{x}_{ij}$, and $v_{j k_j} = g$.
\end{itemize}
\item Else let $\hat{p} = \mbox{argmin}_p d_{ijp}$.
\begin{itemize}
\item If $v_{jc} = \hat{p}$ for some $c$, set $z_{ij} = c$ and $v_{jc} = \hat{p}$.
\item Else, set $k_j = k_j + 1$, $z_{ij} = k_j$, and $v_{j k_j} = \hat{p}$.
\end{itemize}
\end{itemize}
\STATE 5.~For all local clusters:
\begin{itemize}
\item Let $S_{jc} = \{\bm{x}_{ij}|z_{ij} = c\}$ and $\bm{\bar{\mu}_{jc}} = \frac{1}{|S_{jc}|} \sum_{\bm{x} \in S_{jc}} \bm{x}.$
\item Compute $\bar{d}_{jcp} = \sum_{\bm{x} \in S_{jc}} \|\bm{x} - \bm{\mu}_p\|^2$ for $p = 1, ..., g$.
\item If $\min_{p} \bar{d}_{jcp} > \lambda_g + \sum_{\bm{x} \in S_{jc}} \|\bm{x} - \bm{\bar{\mu}}_{jc}\|^2$, set $g = g + 1$, $v_{jc} = g$, and
$\bm{\mu}_g = \bm{\bar{\mu}}_{jc}.$
\item Else set $v_{jc} = \mbox{argmin}_p \bar{d}_{icp}$.
\end{itemize}
\STATE 6.~For each global cluster $p = 1, ..., g$, re-compute means:
\begin{itemize}
\item Let $\ell_{p} = \{\bm{x}_{ij}|z_{ij}=c \mbox{ and }v_{jc}=p\}$.
\item Compute
$\bm{\mu}_p = \frac{1}{|\ell_p|}\sum_{\bm{x} \in \ell_p} \bm{x}.$
\end{itemize}
\end{minipage}
\end{algorithmic}
\end{algorithm}

As with the DP-means algorithm, we can determine the underlying objective function, and use it to determine convergence.  Let $k = \sum_{j=1}^D k_j$ be the total number of local clusters, and $g$ be the total number of global clusters.  Then we can show that the objective optimized is the following:
\begin{eqnarray}
\min_{\{\ell_p\}_{p=1}^g} & \sum_{p=1}^g \sum_{\bm{x}_{ij} \in \ell_p} \|\bm{x}_{ij} - \bm{\mu}_p\|_2^2 + \lambda_{\ell} k + \lambda_g g, \nonumber \\
\mbox{where} & \bm{\mu}_p = \frac{1}{|\ell_p|} \sum_{\bm{x}_{ij} \in \ell_p} \bm{x}_{ij}
\label{eqn:hdp_obj}
\end{eqnarray}
This objective is pleasantly simple and intuitive: we minimize the global k-means objective function, but we incorporate a penalty whenever either a new local cluster or a new global cluster is created.  With appropriately chosen $\lambda_{\ell}$ and $\lambda_g$, the result is that we obtain sharing of cluster structure across data sets.  We can prove that the hard Gaussian HDP algorithm monotonically minimizes this objective (the proof is similar to Theorem 3.1).
\begin{theorem}
Algorithm~\ref{algo:nphmeans} monotonically minimizes the objective~\eqref{eqn:hdp_obj} until local convergence.
\end{theorem}

\section{Further Extensions}
We now discuss two additional extensions of the proposed objective: a spectral relaxation for the proposed hard clustering method and a normalized cut algorithm that does not fix the number of clusters in the graph.

\subsection{Spectral Meets Nonparametric}
Recall that spectral clustering algorithms for k-means are based on the observation that the k-means objective can be relaxed to a problem where the globally optimal solution may be computed via eigenvectors.  In particular, for the k-means objective, one computes the eigenvectors corresponding to the $k$ largest eigenvalues of the kernel matrix $K$ over the data; these eigenvectors form the globally optimal ``relaxed" cluster indicator matrix~\cite{zha_spectral}.  A clustering of the data is obtained by suitably post-processing the eigenvectors, e.g., clustering via k-means.

In a similar manner, in this section we will show that the globally optimal solution to a relaxed DP-means objective function is obtained by computing the eigenvectors of the kernel matrix corresponding to all eigenvalues greater than $\lambda$, and stacking these into a matrix.  To prove the correctness of this relaxation, let us denote $Z$ as the $n \times k$ cluster indicator matrix whose rows correspond to the cluster indicator variables $z_{ic}$.  Let $Y = Z (Z^T Z)^{-1/2}$ be a normalized indicator matrix, and notice that $Y^T Y = I$.  We can prove the following lemma.
\begin{lemma}
The DP-means objective function can equivalently be written as
$\max_Y \tr(Y^T (K - \lambda I) Y),$
where the optimization is performed over the space of all normalized indicator matrices $Y$.
\end{lemma}
\begin{proof}
It was shown in~\cite{zha_spectral} that the standard k-means function can be expressed as a trace maximization of $\tr(Y^T K Y)$, over the space of normalized indicator matrices $Y$.  Noting that $\tr(Y^T (\lambda I) Y) = \lambda k$ as $Y$ is an orthogonal $n \times k$ matrix, the lemma follows.
\end{proof}
Now we perform a standard spectral relaxation by relaxing the optimization to be over all orthonormal matrices $Y$:
\begin{equation}
\max_{\{Y~|~ Y^TY = I\}} \tr(Y^T (K - \lambda I) Y).
\label{eqn:relax}
\end{equation}
Using standard arguments, one can show the following result:
\begin{theorem}
By relaxing the cluster indicator matrix $Y$ to be any orthonormal matrix, the optimal $Y$ in the relaxed clustering objective~\eqref{eqn:relax} is obtained by forming a matrix of all eigenvectors of $K$ whose corresponding eigenvalues are greater than $\lambda$.
\end{theorem}
Using this result, one can design a simple spectral algorithm that computes the relaxed cluster indicator matrix $Y$, and then clusters the rows of $Y$, as is common for spectral clustering methods.  Thus, the main difference between a standard spectral relaxation for k-means and the DP-means is that, for the former, we take the top-$k$ eigenvectors, while for the latter, we take all eigenvectors corresponding to eigenvalues greater than $\lambda$.

\subsection{Graph Clustering}
It is also possible to develop extensions to the DP-means algorithm for graph cut problems such as normalized and ratio cut.  

We briefly review connections between k-means and graph clustering.  
We first note that it is straightforward to apply k-means in \textit{kernel space}.  Suppose the data has been mapped via some function $\psi$ so that the data points in feature space are $\psi(\bm{x}_1), \psi(\bm{x}_2), ..., \psi(\bm{x}_n)$.  Further suppose we can compute a kernel function $\kappa(\bm{x}_i,\bm{x}_j) = \psi(\bm{x}_i)^T \psi(\bm{x}_j)$ without explicitly computing the function $\psi$.  Note that the computation between a data point and any cluster mean can be expressed in kernel space by expanding the squared Euclidean distance:
\begin{eqnarray*}
\|\psi(\bm{x}) - \bm{\mu}_c\|_2^2 & = & (\psi(\bm{x}) - \bm{\mu}_c)^T (\psi(\bm{x}) - \bm{\mu}_c)\\
& = & \psi(\bm{x})^T \psi(\bm{x}) - 2 \psi(\bm{x})^T \bm{\mu}_c + \bm{\mu}_c^T \bm{\mu}_c\\
& = & \psi(\bm{x})^T \psi(\bm{x}) - \frac{2 \sum_{\psi(\bm{x}_i) \in \ell_c} \psi(\bm{x})^T \psi(\bm{x}_i)}{|\ell_c|} + \frac{\sum_{\psi(\bm{x}_i),\psi(\bm{x}_j) \in \ell_c} \psi(\bm{x}_i)^T \psi(\bm{x}_j)}{|\ell_c|^2}\\
& = & \kappa(\bm{x},\bm{x}) - \frac{2 \sum_{\psi(\bm{x}_i) \in \ell_c} \kappa(\bm{x}, \bm{x}_i)}{|\ell_c|} + \frac{\sum_{\psi(\bm{x}_i),\psi(\bm{x}_j) \in \ell_c} \kappa(\bm{x}_i,\bm{x}_j)}{|\ell_c|^2}
\end{eqnarray*}
When computing distances via the above method, it is unnecessary to explicitly compute the means of the clusters.  Therefore, the resulting kernel k-means algorithm does not have an explicit mean re-estimation step; instead, the distances to each implicit cluster mean are computed in kernel space and then each point is re-assigned to the cluster corresponding to the nearest implicit cluster mean.  This is repeated until convergence.  Note that the DP-means algorithm and the hard Gaussian HDP can also easily be applied in kernel space.

Another extension of the k-means objective is to introduce a weight $w_i$ for each point, and to minimize a weighted form of the k-means objective function:
\begin{eqnarray*}
\min_{\{\ell_c\}_{c=1}^k} & \sum_{c=1}^k \sum_{\bm{x} \in \ell_c} w_i \|\bm{x}_i - \bm{\mu}_c\|_2^2\\
\mbox{where} & \bm{\mu}_c = \frac{\sum_{\bm{x}_i \in \ell_c} w_i \bm{x}_i}{\sum_{\bm{x}_i \in \ell_c} w_i}.
\end{eqnarray*}
We can now make a connection between the k-means objective function and \textit{graph clustering}~\cite{ratio_cut,norm_cut,multiclass_ncut}.  Given a graph $G = ({\mathcal V}, {\mathcal E}, A)$, where ${\mathcal V}$ is a set of vertices, ${\mathcal E}$ a set of edges, and $A$ is the underlying adjacency matrix for the graph, various methods have been proposed to cluster the vertices in the graph into a disjoint collection of clusters.  Two popular criteria for the graph clustering problem are the ratio cut and normalized cut.  In the ratio cut, one seeks a clustering ${\mathcal V_1}, ..., {\mathcal V_k}$ of the vertices to minimize the following objective:
\begin{displaymath}
\min_{{\mathcal V_1}, ..., {\mathcal V_k}} \sum_{c=1}^k \frac{\mbox{links}({\mathcal V_c}, {\mathcal V} \setminus {\mathcal V_c})}{|{\mathcal V_c}|},~~~~\mbox{(Ratio Cut)}
\end{displaymath}
where $\mbox{links}({\mathcal B}, {\mathcal C}) = \sum_{i \in {\mathcal B}, j \in {\mathcal C}} A_{ij}.$  Thus, the ratio cut criterion attempts to find clusters of vertices such that the ``cut" from clusters to remaining nodes in the graph (normalized by the size of the clusters), is minimized.  The related normalized cut problem minimizes the following:
\begin{displaymath}
\min_{{\mathcal V_1}, ..., {\mathcal V_k}} \sum_{c=1}^k \frac{\mbox{links}({\mathcal V_c}, {\mathcal V} \setminus {\mathcal V_c})}{\mbox{deg}({\mathcal V_c})},~~~~\mbox{(Normalized Cut)}
\end{displaymath}
where $\mbox{deg}({\mathcal B}) = \mbox{links}({\mathcal B},{\mathcal V})$ (or equivalently, the sum of the degrees of the vertices in ${\mathcal B}$).

We state a result proven in~\cite{dhillon} for standard k-way normalized cut.
\begin{theorem}
Let $J(K,W)$ be the weighted kernel k-means objective with kernel matrix $K$ and (diagonal) weight matrix $W$, and let $Cut(A)$ be the k-way normalized cut objective with adjacency matrix $A$.  Let $D$ be the diagonal degree matrix corresponding to $A$ ($D = \mbox{diag}(A \bm{e})$).  Then the following relationship holds:
\begin{displaymath}
J(K,W) = \sigma n + \tr(D^{-1/2} A D^{-1/2}) - (\sigma + 1) k + Cut(A),
\end{displaymath}
when we define $K = \sigma D^{-1} + D^{-1} A D^{-1}$, $W = D$, and $\sigma$ is large enough that $K$ is positive semi-definite.  A similar result holds for ratio cut.
\end{theorem}
Let the DP-means objective---easily extended to kernel space and to use weights---be given by $J(K,W) + \lambda k$, and let the analogous penalized normalized cut objective be given by $Cut(A) + \lambda' k$.  Letting $\sigma n + \tr(D^{-1/2} A D^{-1/2}) = C$, a constant, we have that $J(K,W) + \lambda k = $
\begin{displaymath}
C + Cut(A) - (\sigma+1)k + \lambda k = C + Cut(A) + \lambda' k,
\end{displaymath}
where $\lambda' = \lambda - \sigma - 1$.  Thus, optimizing the hard DP weighted kernel k-means objective with model parameter $\lambda$ is equivalent to optimizing the penalized normalized cut objective with model parameter $\lambda' = \lambda - \sigma - 1$, and with the construction of $K$ and $W$ as in the above theorem.  Utilizing the results of~\cite{dhillon}, one can show that the distance between a node and a cluster mean can be performed in $O(|E|)$ time.  A straightforward extension of Algorithm~\ref{algo:npmeans} can then be adapted for the above penalized normalized cut objective.

\section{Experiments}

\begin{figure*}[t]
\centering
\begin{tabular}{cccc}
\includegraphics[width=3.6cm,height=3.1cm]{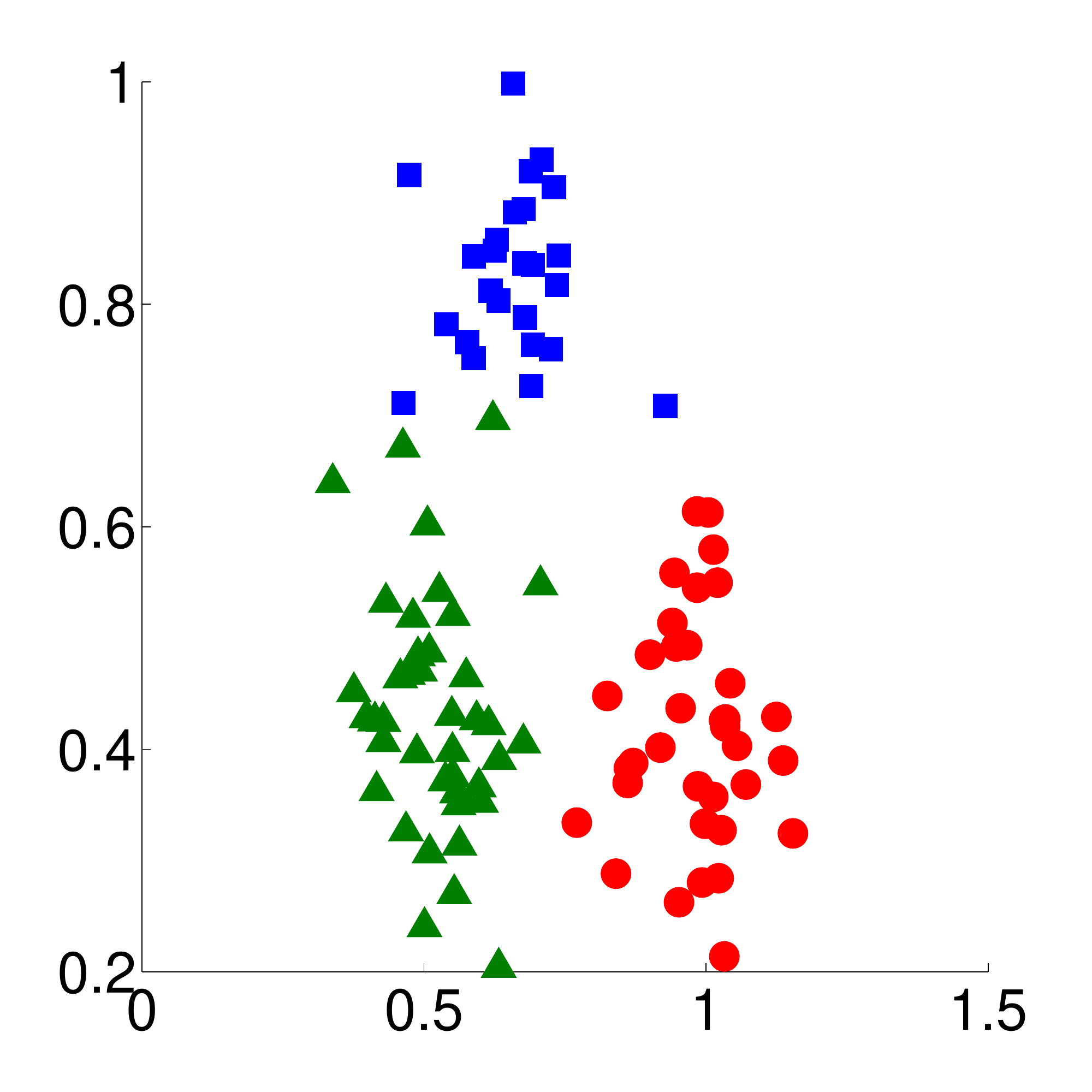} &
\includegraphics[width=3.6cm,height=3.1cm]{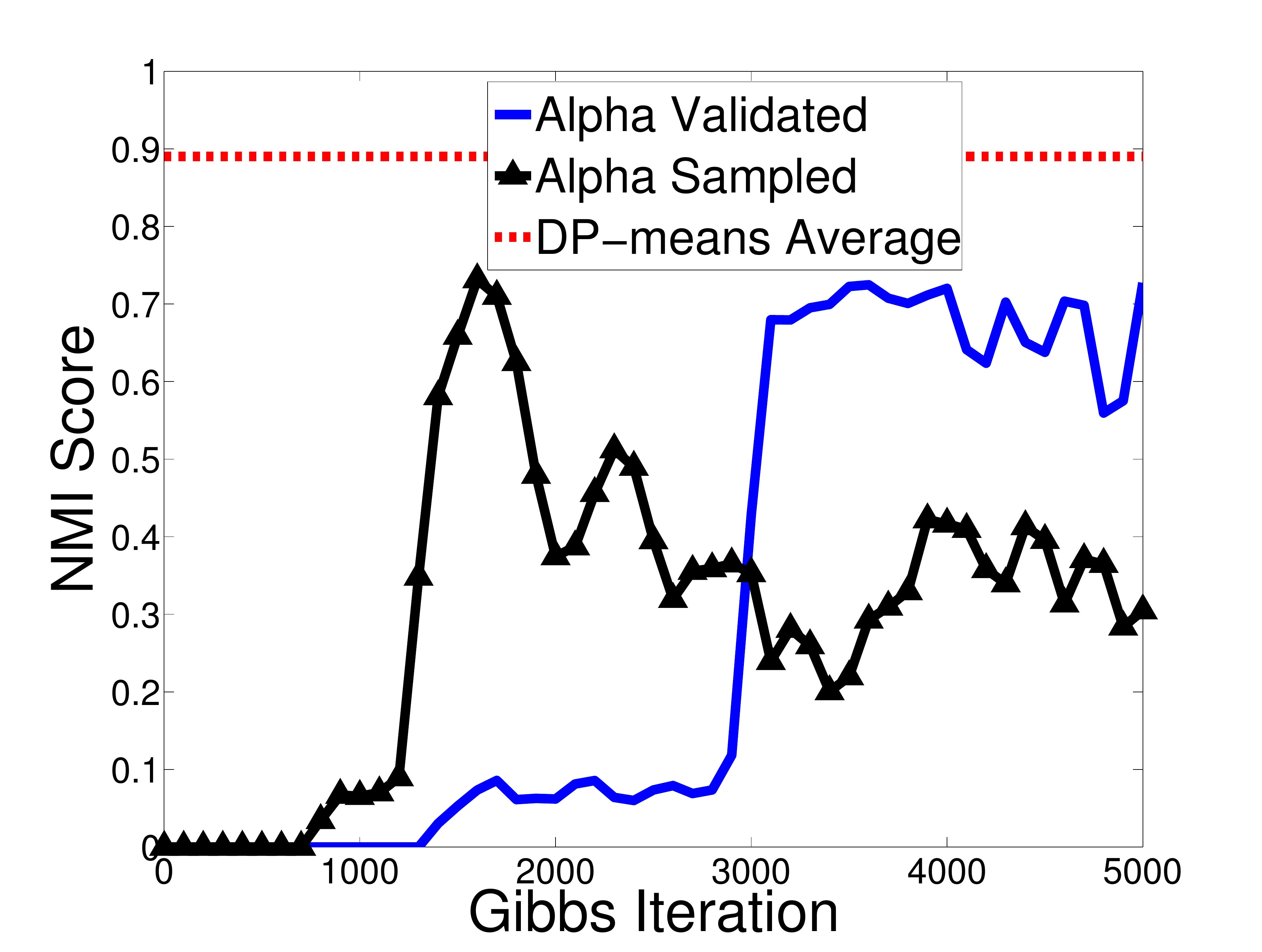} &
\includegraphics[width=3.6cm,height=3.1cm]{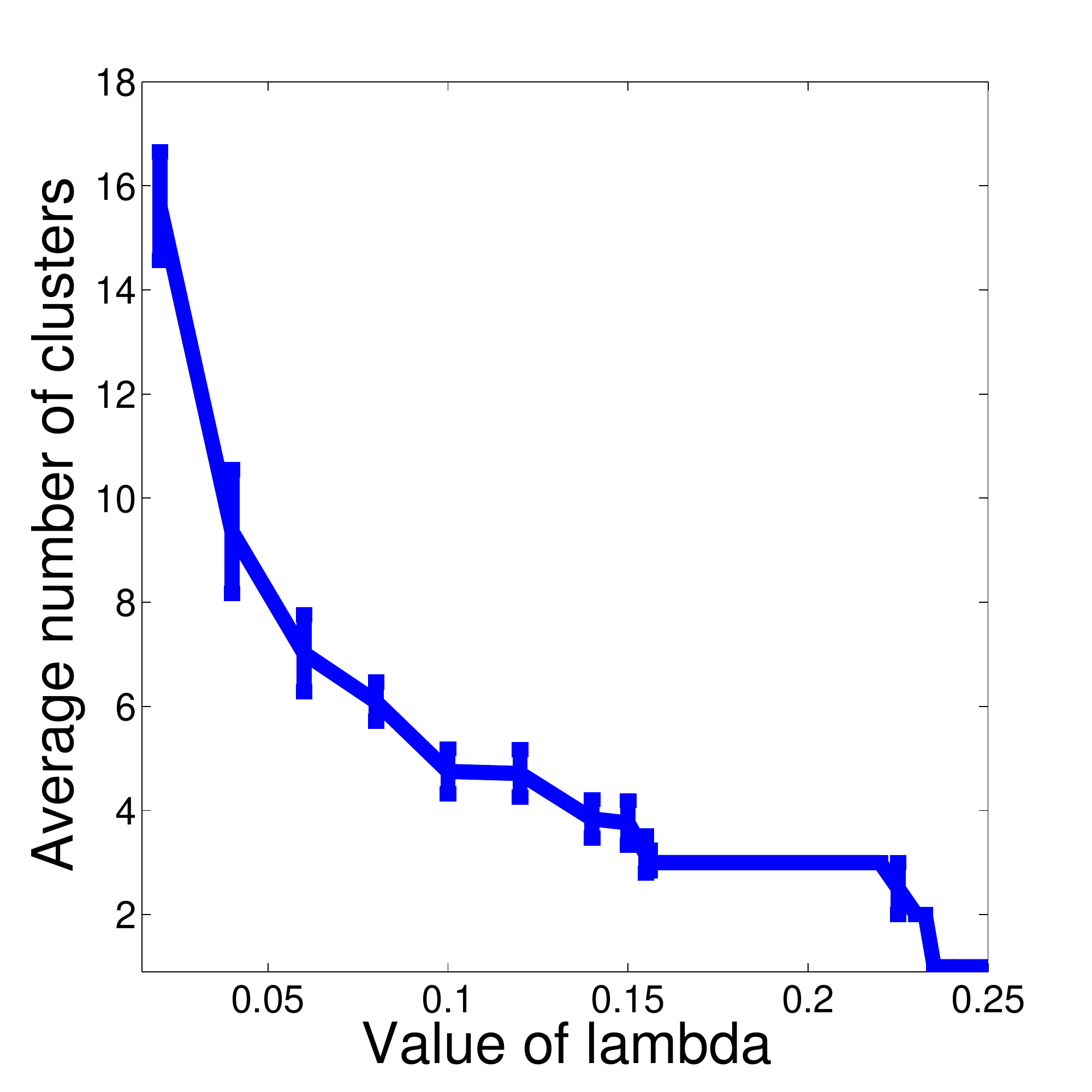} &
\includegraphics[width=3.6cm,height=3.1cm]{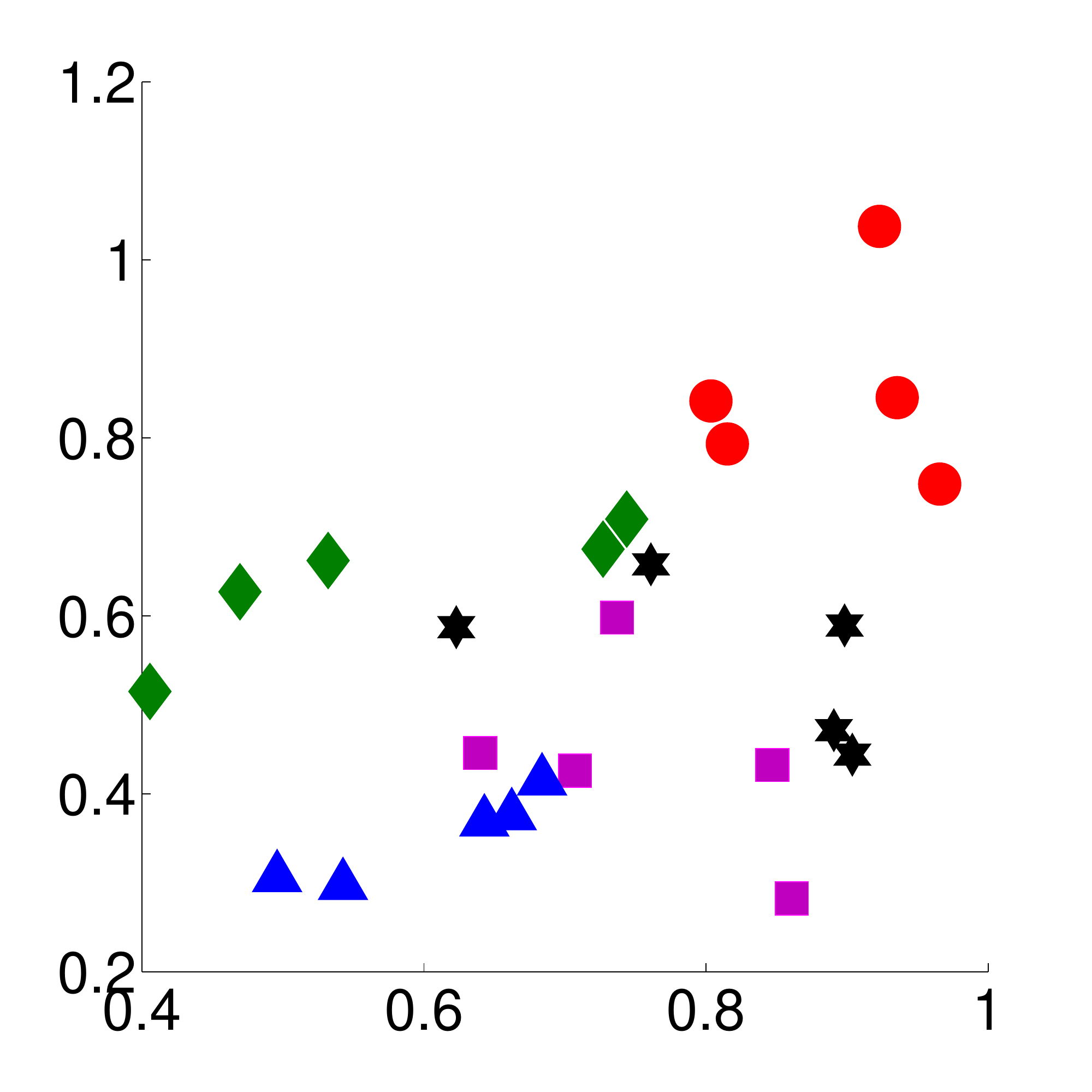}
\end{tabular}
\caption{Synthetic results demonstrating advantages of our method.  a) A simple data set of 3 Gaussians.  b) NMI scores over the first 5000 Gibbs iterations---in contrast, across 100 runs, DP-means always converges within 8 iterations on this data, always returns 3 clusters, and yields an average NMI of .89.  c) Number of clusters in DP-means as a function of lambda. d) One of the 50 data sets for the hard Gaussian HDP experiment.  See text for details.}
\label{fig:gauss}
\end{figure*}

\begin{table}
\centering
\begin{tabular}{|c||c|c|c|}
\hline
Data set & DP-means & k-means & Gibbs\\
\hline \hline
Wine & .41 & .43 & .42\\
\hline
Iris & .75 & .76 & .73\\
\hline
Pima & .02 & .03 & .02\\
\hline
Soybean & .72 & .66 & .73\\
\hline
Car & .07 & .05 & .15\\
\hline
Balance Scale & .17 & .11 & .14\\
\hline
Breast Cancer & .04 & .03 & .04\\
\hline
Vehicle & .18 & .18 & .17\\
\hline
\end{tabular}
\caption{Average NMI scores on a set of UCI data sets.  Note that Gibbs sampling is handicapped since we utilize a validation set for parameter tuning.}
\label{tab:uci_results}
\end{table}

We conclude with a brief set of experiments to demonstrate the utility of our approach.  The goal is to demonstrate that hard clustering via Bayesian nonparametrics enjoys many properties of Bayesian techniques (unlike k-means) but features the speed and scalability of k-means. 

\subsection{Setup}
Throughout the experiments, we utilize normalized mutual information (NMI) between ground truth clusters and algorithm outputs for evaluation, as is common for clustering applications (it also allows us to compare results when the number of outputted clusters does not match the number of clusters in the ground truth).
Regarding parameter selection, there are various potential ways of choosing $\lambda$; for clarity in making comparisons to k-means we fix $k$ (and $g$) and then find a suitable $\lambda$.  In particular, we found that a simple farthest-first heuristic is effective, and we utilize this approach in all experiments.  Given an (approximate) number of desired clusters $k$, we first initialize a set $T$ with the global mean.  We iteratively add to $T$ by finding the point in the data set which has the maximum distance to $T$ (the distance to $T$ is the smallest distance among points in $T$).  We repeat this $k$ times and return the value of the maximum distance to $T$ in round $k$ as $\lambda$.  We utilize a similar procedure for the hard HDP, except that for $\lambda_{\ell}$ we average the values of the above procedure over all data sets, and for $\lambda_g$ we replace distances of points to elements of $T$ with sums of distances of points in a data set to elements of $T$.  
For Gibbs sampling, we consider the model where the covariances are fixed to $\sigma I$, there is a zero-mean $\rho I$ Gaussian prior on the means, and an inverse-Gamma prior on $\sigma$.  (For the benchmark data, we considered selection of $\sigma$ based on cross-validation, as it yielded better results, though this is against the Bayesian spirit.)  We set $\rho = 100$ throughout our experiments.  We also consider two strategies for determining $\alpha$: one where we place a gamma prior on $\alpha$, as is standard for DP mixtures~\cite{escobar}, and another where we choose $\alpha$ via a validation set.

\subsection{DP-means Results}
We begin with a simple illustration of some of the key properties of our approach on a synthetic data set of three Gaussians, shown in Figure~\ref{fig:gauss}a.  When we run DP-means on this data, the algorithm terminates within 8 iterations with an average NMI score of .89 (based on 100 runs).  In contrast, Figure~\ref{fig:gauss}b shows the NMI scores of the clusterings produced by two Gibbs runs (no burn-in) over the first 5000 iterations, one that learns $\alpha$ via a gamma prior, and another that uses a validation set to tune $\alpha$.  The learning approach does well around 1500 iterations, but eventually more than three clusters are produced, leading to poor results on this data set.  The validation approach yields three clusters, but it takes approximately 3000 iterations before Gibbs sampling is able to converge to three clusters (in contrast, it typically requires only three iterations before DP-means has reached an NMI score above .8).  Additionally, we plot the number of clusters produced by DP-means as a function of $\lambda$ in Figure~\ref{fig:gauss}c; here we see that there is a large interval of $\lambda$ values where the algorithm returns three clusters.  Note that all methods are initialized with all points in a single cluster; we fully expect that better initialization would benefit these algorithms.

Next we consider a benchmarking comparison among k-means, DP-means, and Gibbs sampling to demonstrate comparable accuracies among the methods.  
We selected 8 common UCI data sets, and used class labels as the ground-truth for clusters.  Each data set was randomly split 30/70 for validation/clustering (we stress that validation is used only for Gibbs sampling).  On the validation set, we validated both $\alpha$ and $\sigma$, which yielded the best results.  We ran the Gibbs sampler for 1000 burn-in iterations, and then ran for 5000 iterations, selecting every 10 samples.  The NMI is computed between the ground-truth and the computed clusters, and results are averaged over 10 runs.
The results are shown in Table~\ref{tab:uci_results}.  We see that, as expected, the results are comparable among the three algorithms: DP-means achieves higher NMI on 5 of 8 data sets in comparison to k-means, and 4 of 8 in comparison to Gibbs sampling.  

To demonstrate scalability, we additionally ran our algorithms over the 312,320 image patches of the Photo Tourism data set~\cite{photo_tourism}, a common vision data set.  Each patch is 128-dimensional.  Per iteration, the DP-means algorithm and the Gibbs sampler require similar computational time (37.9 seconds versus 29.4 seconds per iteration).  However, DP-means converges fully in 63 iterations, whereas obtaining full convergence of the Gibbs sampler is infeasible on this data set.

\subsection{Hard Gaussian HDP Results}
As with DP-means, we demonstrate results on synthetic data to highlight the advantages of our approach as compared to the baselines.  We generate parameters for 15 ground-truth Gaussian distributions (means are chosen uniformly in $[0,1]^2$ and covariances are $.01 \cdot I$).  Then we generate 50 data sets as follows: for each data set, we choose 5 of the 15 Gaussians at random, and then generate 25 total points from these chosen Gaussians (5 points per Gaussian).  An example of one of the 50 data sets is shown in Figure~\ref{fig:gauss}d; in many cases, it is difficult to cluster the data sets individually, as shown in the figure.

Our goal is to find shared clusters in this data.  To evaluate the quality of results, we compute the NMI between the ground-truth and the outputted clusters, for each data set, and average the NMI scores across the data sets.  As a baseline, we run k-means and DP-means on the whole data set all at once (i.e., we treat all twenty data sets as one large data set) as well as k-means and DP-means on the individual data sets.  k-means on the whole data set obtains an average NMI score of .77 while DP-means yields .73.  When we run the hard Gaussian HDP, we obtain 17 global clusters, and each data set forms on average 4.4 local clusters per data set.  The average NMI for this method is .81, significantly higher than the non-hierarchical approaches.  When we run k-means or DP-means individually on each data set and compute the average NMI, we obtain scores of .79 for both; note that there is no automatic cluster sharing via this approach.  The hard Gaussian HDP takes 28.8s on this data set, versus 2.7s for k-means on the full data. 

\section{Conclusions and Open Problems}
This paper outlines connections arising between DP mixture models and hard clustering algorithms, and develops scalable algorithms for hard clustering that retain some of the benefits of Bayesian nonparametric and hierarchical modeling.  Our analysis is only a first step, and we note that there are several avenues of future work, including i) improvements to the basic algorithms using ideas from k-means, such as local search~\cite{localsearch}, ii) spectral or semidefinite relaxations for the hard Gaussian HDP, 
iii) extensions to other Bayesian nonparametric processes such as the Pitman-Yor process~\cite{ishwaran,pitman},
iv) generalizations to exponential family mixture models~\cite{banerjee}, and v) additional comparisons to sampling-based and variational inference methods.
\\

\noindent \textbf{Acknowledgements.}  We thank Trevor Darrell and the anonymous reviewers for the corresponding ICML paper for helpful suggestions.

{\small
\bibliography{np_clustering}
\bibliographystyle{plain}
}
\end{document}